\documentclass[11pt]{elsarticle}
\usepackage[T1]{fontenc}
\usepackage[utf8]{inputenc}
\usepackage{lmodern}
\usepackage{microtype}
\usepackage[a4paper,margin=1in]{geometry}
\usepackage{amsmath,amssymb,amsthm}
\usepackage{booktabs,tabularx,array}
\newcolumntype{L}{>{\raggedright\arraybackslash}X}
\newcolumntype{P}[1]{>{\raggedright\arraybackslash}p{#1}}
\usepackage{graphicx}
\usepackage{xcolor}
\usepackage{hyperref}
\usepackage{textgreek}
\usepackage{xspace}
\usepackage{enumitem}
\usepackage{siunitx}
\usepackage{caption}
\usepackage{subcaption}
\usepackage{float}
\usepackage{placeins}
\usepackage{multirow}
\usepackage{longtable}
\hypersetup{
    colorlinks=true,
    linkcolor=blue!50!black,
    citecolor=blue!50!black,
    urlcolor=blue!50!black,
    pdftitle={Auxiliary Finite-Difference Residual-Gradient Regularization for PINNs},
    pdfauthor={Stavros Kassinos},
}
\newtheorem{proposition}{Proposition}
\newtheorem{remark}{Remark}
\newcommand{\R}{\mathbb{R}}

\newcommand{\grad}{\nabla}

\newcommand{\Lpde}{\mathcal{L}_{\mathrm{PDE}}}
\newcommand{\Lbc}{\mathcal{L}_{\mathrm{BC}}}

\newcommand{\Ltot}{\mathcal{L}}
\newcommand{\Lfd}{\mathcal{L}_{\mathrm{FD\text{-}RG}}}
\newcommand{\Lad}{\mathcal{L}_{\mathrm{AD\text{-}RG}}}

\newcommand{\kourbeta}{Kourkoutas-\textbeta\xspace}

\begin{document}
\begin{frontmatter}

\title{\textbf{Auxiliary Finite-Difference Residual-Gradient Regularization for PINNs}\\
\large Controlled Poisson evidence and a body-fitted 3D shell application}
\author{Stavros Kassinos\corref{cor1}}
\cortext[cor1]{Corresponding author, kassinos.stavros@ucy.ac.cy}

\affiliation{organization={Computational Sciences Laboratory, Department of Mechanical Engineering, University of Cyprus}, 
            addressline={\\ 1 University Avenue}, 
            city={Aglantzia},
            postcode={2109}, 
            state={Nicosia},
            country={Cyprus}
            }

\date{Manuscript draft, April 2026}

\begin{abstract}
Physics-informed neural networks (PINNs) are often selected by a single scalar loss even when the quantity of interest is more specific. This paper studies a hybrid design in which the governing PDE residual remains automatic-differentiation (AD) based, while finite differences (FD) appear only in an auxiliary term that penalizes gradients of the sampled residual field. Unlike AD residual-gradient penalties, FD/ND substitutions inside the main physics loss, or fully discrete residual objectives, the FD term regularizes the residual field without replacing the PDE residual itself.

We examine the idea in two stages. Stage~1 is a controlled study on a manufactured two-dimensional Poisson problem, comparing a baseline PINN, the FD residual-gradient regularizer, and a matched AD residual-gradient baseline. The FD term reproduces the main effect of residual-gradient control while exposing a clear trade-off between field accuracy and residual cleanliness. Stage~2 transfers the same residual-field logic to a three-dimensional annular heat-conduction benchmark (PINN3D), where baseline errors concentrate near a wavy outer wall; there the auxiliary grid is implemented as a body-fitted shell adjacent to the wall.

In Stage~2, the shell regularizer improves the application-facing quantities, namely outer-wall flux and boundary-condition behavior. Across seeds~0--5 and 100k epochs, the most reliable tested configuration is a fixed shell weight of $5\times 10^{-4}$ under the \kourbeta{} optimizer regime: relative to a matched run without the shell term, it reduces the mean outer-wall BC RMSE from $1.22\times 10^{-2}$ to $9.29\times 10^{-4}$ and the mean wall-flux RMSE from $9.21\times 10^{-3}$ to $9.63\times 10^{-4}$. A linearly scheduled shell weight helps on some seeds but is less consistent. Optimizer and learning-rate choice also matter: under the baseline cosine schedule with initial learning rate $7.5\times 10^{-3}$, Adam with $\beta_2=0.95$ or $\beta_2=0.999$ is unreliable. Adam with $\beta_2=0.999$ becomes usable when the initial learning rate is reduced to $10^{-3}$, but its shell benefit is less robust than under \kourbeta{}. Overall, the results support a targeted view of hybrid PINNs: an auxiliary-only FD regularizer is most valuable when it is aligned with the physical quantity of interest, here the outer-wall flux.
\end{abstract}

\end{frontmatter}

\section{Introduction}

A physics-informed neural network is asked to satisfy several scientific requests at once. It should reduce a PDE residual, honor multiple boundary conditions, remain periodic or symmetric when required, and sometimes fit data or auxiliary constraints as well. The resulting scalar loss is useful, but it is only a compressed summary of a negotiation among these terms. A model can look satisfactory under one scalar objective while still being poor for the quantity that ultimately matters in the application, for example a wall flux or a boundary-condition residual.

That practical mismatch is the starting point of this paper. The motivating PINN3D annular benchmark \cite{Kassinos2025KBeta, Kassinos2025PINN3D} already had a clear pain point before the present work began: the baseline PINN was most fragile near the undulating outer wall, and the quantity of real interest there was not a generic loss value but the wall-normal heat flux and the associated boundary residual. This raised a natural question. If the governing PDE itself only needs low-order derivatives, can we add a weak structured prior that improves the \emph{residual field} in the region where the baseline model is weakest, without replacing the underlying continuous PINN formulation?

Our answer is deliberately focused. We keep the main PDE residual in automatic-differentiation (AD) form, sample that residual on a structured auxiliary grid, and apply finite differences (FD) only to the sampled residual field. The auxiliary term is therefore not a discrete replacement for the PDE residual. It is a structured regularizer on how the already-defined residual field varies in space. In Stage~1 this idea is tested in a clean manufactured Poisson problem. In Stage~2 the same logic is transferred to a body-fitted outer shell wrapped around the PINN3D annulus.

A reader coming from outside the PINN literature may want one sentence of terminology up front. We use \emph{automatic differentiation} (AD) for the exact derivative operators provided by the computational graph of the neural network, and \emph{finite differences} (FD) for the discrete derivative operators used on a structured auxiliary grid. The method in this paper is hybrid in a very specific sense: the governing residual remains continuous and AD-based, while the auxiliary regularizer is FD-based.

The prior literature already constrains how broad the novelty claim can be. Standard PINNs are well established \cite{raissi2019pinn}; gradient-enhanced variants already use derivatives of the residual \cite{yu2022gpinn}; and hybrid AD/FD or AD/ND ideas are also present in the literature \cite{xiang2022hfdpinn,chiu2022canpinn}. The present paper therefore makes a narrower claim. The governing residual remains continuous and AD-based, while finite differences are introduced only through an auxiliary regularizer acting on the \emph{sampled} residual field. In Stage~2 that auxiliary term is further localized to a body-fitted shell near the boundary quantity of interest. Section~\ref{sec:related-work} positions this design more carefully relative to the nearest prior work.

What is new here is the combination of that specific hybrid formulation with a controlled two-stage study. Stage~1 asks a mechanism question: does a structured FD residual-gradient term change training in a meaningful and interpretable way, how does it compare with a matched AD residual-gradient baseline, and is it merely sensitive to one auxiliary-grid phase? Stage~2 asks an application question: can the same residual-field logic, deployed on a body-fitted outer shell, improve the outer-wall flux behavior of a realistic PINN3D annular benchmark? These two stages play different roles, and the manuscript is organized to keep that distinction explicit.

Because the Stage~2 benchmark and model family are inherited from prior PINN3D/\kourbeta{} work, it is useful to separate what is reused from what is new. Table~\ref{tab:inheritance-new} does that explicitly.

\begin{table}[H]
\centering
\caption{Benchmark inheritance versus new contribution. The present manuscript reuses the PINN3D annular benchmark from earlier work, but adds a new structured residual-gradient regularizer together with a two-stage empirical study designed to separate mechanism from application.}
\label{tab:inheritance-new}
\small
\begin{tabularx}{\linewidth}{P{0.47\linewidth} P{0.47\linewidth}}
\toprule
\textbf{Inherited benchmark ingredients} & \textbf{New in this manuscript} \\
\midrule
Steady PINN3D annular heat-conduction workload, wavy outer-wall geometry, six-term weighted objective, reference FD solution, and the base MLP/optimizer context from the prior PINN3D and \kourbeta{} work \cite{Kassinos2025KBeta,Kassinos2025PINN3D}. & An auxiliary-only FD regularizer applied to the \emph{sampled AD residual field}; the controlled Stage-1 Poisson mechanism study; the body-fitted outer-shell deployment in Stage~2; seedwise wall-reference audits; and the optimizer and learning-rate (LR) sensitivity checks used to interpret the Stage-2 result carefully. \\
\bottomrule
\end{tabularx}
\end{table}

The paper is organized to separate controlled mechanism evidence from application evidence. Stage~1 serves as the controlled mechanism study, whereas Stage~2 provides the body-fitted application case. Throughout, the notation is kept as light as possible, and the metric hierarchy is made explicit whenever the scalar objective and the physically meaningful wall quantities disagree.

\subsection*{Contributions}
The paper makes four concrete contributions.
\begin{enumerate}[leftmargin=1.5em]
    \item It formulates an auxiliary-only finite-difference regularizer on the sampled residual field, while keeping the governing PDE residual continuous and AD-based.
    \item It evaluates that regularizer in a controlled Stage-1 benchmark against matched AD residual-gradient comparators, anti-grid-lock audits, and schedule ablations.
    \item It lifts the same residual-field logic to a body-fitted outer shell on a nontrivial 3D annular geometry without extending the auxiliary grid outside the physical domain.
    \item It shows that in Stage~2 the fixed shell regularizer robustly improves outer-wall flux and boundary-condition behavior across six seeds in the main \kourbeta{} regime, while a broader optimizer and learning-rate study clarifies that the effect is portable but less robust under Adam999.
\end{enumerate}

\section{Benchmark problems and evaluation targets}\label{sec:benchmarks}

\subsection{Stage 1: controlled manufactured Poisson benchmark}

Stage~1 uses a manufactured Poisson problem on the unit square. Its role is not to mimic a difficult engineering workload, but to give a clean mechanism study in which exact field values and exact derivatives are available everywhere. This allows us to compare plain PINNs, FD residual-gradient regularizers, and matched AD residual-gradient comparators under dense off-grid audits rather than only at the training cloud.

\subsection{Stage 2: PINN3D annulus and outer-shell target}

Stage~2 reuses the steady PINN3D annular heat-conduction benchmark already developed in the broader \kourbeta{} / PINN3D line of work \cite{Kassinos2025KBeta,Kassinos2025PINN3D,HybridRepo2026}. The physical domain is a finite cylindrical annulus whose outer wall undulates azimuthally,
\[
\Omega = \{(r,\theta,z): r_{\min}\le r\le r_o(\theta),\ 0\le \theta\le 2\pi,\ 0\le z\le L\},
\qquad
r_o(\theta)=r_{\max}+0.25\,r_{\max}\sin(3\theta),
\]
with $r_{\min}=0.2$, $r_{\max}=1.0$, and $L=10r_{\max}$. The inner wall and inlet are held at unit temperature, the outlet is axially insulated, the outer wall carries a prescribed heat flux, and periodicity is enforced in $\theta$.
The prescribed outer-wall Neumann data are not spatially uniform: the imposed heat flux is a piecewise ramped function of the axial coordinate $z$. As a result, the wall quantity of interest carries genuine axial structure, and the wall diagnostics reported later are not simply measuring agreement with a constant-flux boundary.

The inherited PINN3D objective is a weighted six-term objective built from the PDE residual, inner-wall, inlet, outlet, outer-wall, and periodicity terms. The network itself is the same fully-connected model used in the PINN3D benchmark line: a 16-layer multilayer perceptron (MLP) of width 128 with SiLU activations and scalar temperature output. In the steady case, the input is $(r,\theta,z)$ and the output is $T_\theta(r,\theta,z)$. We keep that inherited benchmark structure because the purpose of the present paper is not to redesign PINN3D from scratch, but to test whether a structured residual-field regularizer helps where the baseline model is weakest.

\begin{figure}[h]
    \centering
    \includegraphics[width=0.94\textwidth]{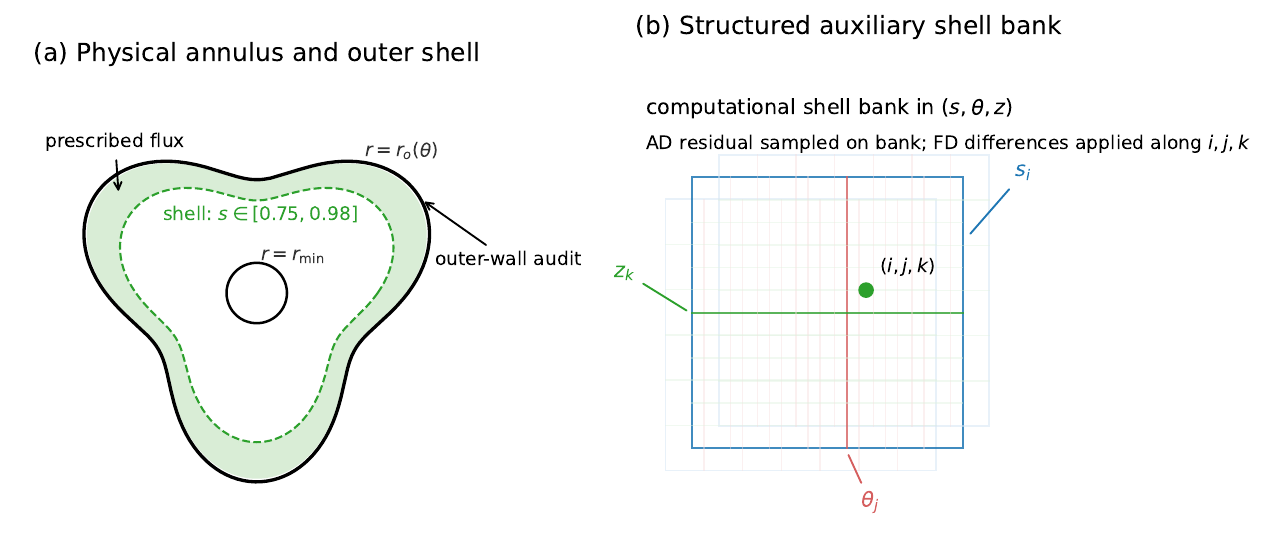}
    \caption{Stage-2 problem geometry and auxiliary shell. Left: cross-section of the annular domain with wavy outer wall $r_o(\theta)$ and the shaded body-fitted shell used by the FD regularizer. Right: structured shell bank in computational coordinates $(s,\theta,z)$. The indices $(i,j,k)$ refer respectively to radial shell line, azimuthal node, and axial node.}
    \label{fig:stage2-geometry}
\end{figure}

For the main Stage-2 experiments, the training cloud contains 4000 interior points, 2000 points on each of the inner/inlet/outlet/outer boundaries, and 400 periodic point pairs. The validation cloud uses the lighter but analogous split already built into the solver: 1024 interior points, 512 points per boundary, and 200 periodic pairs. The shell regularizer is evaluated on a separate body-fitted shell bank, not on the collocation cloud itself. Figure~\ref{fig:stage2-geometry} shows both the physical annulus and the shell-grid idea. The shell is defined as a constant band in the normalized radial coordinate $s$, not as a constant Euclidean-thickness offset from the wall. Its physical thickness therefore scales with the local annular gap $r_o(\theta)-r_{\min}$, where $r_o(\theta)=r_{\max}+0.25\,r_{\max}\sin(3\theta)$. Table~\ref{tab:stage2-protocol} summarizes the inherited PINN3D protocol in the spirit of the \kourbeta{} paper \cite{Kassinos2025KBeta}.

Stage~2 uses three complementary wall-facing diagnostics, and unlike Stage~1 these should not be read as validation-cloud RMSE values. The scalar training and validation losses are still computed on collocation-style clouds, but the RMSE quantities reported later in Table~\ref{tab:stage2-main}, Figures~\ref{fig:stage2-delta} and~\ref{fig:stage2-seedwise}, and the appendix sweep tables are dense wall diagnostics. The first is the \emph{outer-wall BC audit}, an internal consistency check evaluated on a dense wall grid generated directly from the PINN prediction: it compares the predicted wall-normal derivative $\partial_n T$ with the prescribed outer-wall flux, and its root-mean-square value is reported as the wall BC audit RMSE. The second is the \emph{FD wall reference}, a dense outer-wall slice extracted from an \emph{independent} finite-difference (FD) solution; against that reference we report RMSE values for the wall temperature $T_{\mathrm{wall}}$, the wall-normal derivative $\partial_n T_{\mathrm{wall}}$, and the induced wall boundary-condition residual. The construction of this FD reference solution, together with the grid-resolution checks used to justify it as a numerical anchor, is summarized in ~\ref{app:fd-reference}. The third is the \emph{shell probe}, namely the unweighted value of the shell regularizer before multiplication by its coefficient. The shell probe helps interpret how much residual variation remains inside the shell, but it is not itself a primary model-selection metric.

\begin{table}[t]
\centering
\caption{Stage-2 protocol. The benchmark, architecture, and base objective are inherited from the PINN3D/\kourbeta{} line, while the shell regularizer and the wall-audit logic are the new ingredients introduced here.}
\label{tab:stage2-protocol}
\small
\begin{tabularx}{\linewidth}{P{0.28\linewidth} L}
\toprule
\textbf{Item} & \textbf{Stage-2 protocol used in this paper} \\
\midrule
Physical domain & Steady cylindrical annulus with wavy outer wall $r_o(\theta)=r_{\max}+0.25\,r_{\max}\sin(3\theta)$, with $r_{\min}=0.2$, $r_{\max}=1.0$, and $L=10r_{\max}$ \\
Boundary conditions & Inner wall and inlet at unit temperature; outlet Neumann; prescribed flux on the outer wall; periodicity in $\theta$ \\
Base objective & Six weighted terms: PDE, inner, inlet, outlet, outer, and periodic/theta contributions inherited from the PINN3D benchmark \cite{Kassinos2025KBeta} \\
Model & Fully connected 16-layer MLP of width 128 with SiLU activations and scalar temperature output \\
Training cloud & 4000 interior points; 2000 points on each of inner, inlet, outlet, and outer boundaries; 400 periodic point pairs \\
Validation cloud & 1024 interior points; 512 points on each boundary; 200 periodic point pairs \\
Main shell default & Body-fitted outer shell with $s\in[0.75,0.98]$, $(n_s,n_\theta,n_z)=(8,32,32)$, uniform shell distribution, quadrature weighting, fixed bank \\
Main optimizer regime & \kourbeta{} with \texttt{kbeta\_decay=0.98}; default initial LR $7.5\times 10^{-3}$ with cosine decay to $10^{-5}$ \\
Primary Stage-2 audits & Dense outer-wall BC audit and dense comparison against the FD reference for $T_{\mathrm{wall}}$, $\partial_n T_{\mathrm{wall}}$, and BC residual \\
\bottomrule
\end{tabularx}
\end{table}

\section{Related work and positioning}
\label{sec:related-work}

The literature most relevant to the present paper can be grouped into three layers. First, standard PINNs combine a strong-form PDE residual with boundary and data terms \cite{raissi2019pinn}, and their optimization pathologies are now well documented \cite{wang2021gradpath,wang2022ntk}. Second, gradient-enhanced PINNs augment the loss with derivatives of the PDE residual \cite{yu2022gpinn}. Third, several works have already mixed automatic differentiation with numerical differentiation or finite differences, including HFD-PINN \cite{xiang2022hfdpinn} and CAN-PINN \cite{chiu2022canpinn}. A further theoretical caution is that when the training objective is built directly from a finite-difference residual, the learned solution becomes tied to that discrete scheme \cite{langer2026fdpinn}.

Relative to those priors, the present paper should be read as a focused specialization rather than as a new universal PINN principle. The governing residual remains continuous and AD-based. Finite differences appear only in an auxiliary term, where they act on the sampled residual field rather than replacing the main PDE residual. In Stage~1 this allows a controlled comparison between an FD residual-gradient regularizer and a matched AD residual-gradient baseline. In Stage~2 it allows the same residual-field logic to be localized to a body-fitted shell near the outer-wall quantity of interest.

Among the prior hybrids, HFD-PINN is probably the closest in spirit to the present work. The key distinction is that HFD-PINN uses finite differences locally \emph{instead of} automatic differentiation in parts of the physics loss, whereas the present method keeps the governing residual AD-based and introduces finite differences only through an auxiliary regularizer acting on the sampled residual field. In embedded-surrogate settings, this distinction is practically useful because the auxiliary FD term can be chosen to reflect the stencil family, formal order, and boundary closure of the surrounding discretization without replacing the continuous PINN residual itself.

\section{Method}

\subsection{Base PINN formulation}

Let $u_\theta : \Omega \to \R$ denote the network output and let
\begin{equation}
R_\theta(x)=\mathcal{N}[u_\theta](x)-f(x)
\end{equation}
be the continuous PDE residual evaluated by automatic differentiation. With interior cloud $X_r$ and boundary cloud $X_b$, the baseline PINN objective is
\begin{equation}
\Ltot_{\mathrm{PINN}}(\theta)=\Lpde(\theta)+\lambda_{\mathrm{BC}}\Lbc(\theta),
\end{equation}
with
\begin{equation}
\Lpde(\theta)=\frac{1}{|X_r|}\sum_{x\in X_r}|R_\theta(x)|^2,
\qquad
\Lbc(\theta)=\frac{1}{|X_b|}\sum_{x\in X_b}|B_\theta(x)|^2,
\end{equation}
where $B_\theta$ denotes the appropriate boundary residual.

\subsection{Auxiliary FD residual-gradient regularizer}

Let $X_h=\{x_i\}_{i=1}^{N_h}$ be a structured auxiliary grid. We evaluate the continuous residual on that grid,
\begin{equation}
\mathbf{r}_\theta = \bigl(R_\theta(x_i)\bigr)_{i=1}^{N_h},
\end{equation}
and then apply finite-difference derivative operators to the sampled residual field. In Stage 1 the auxiliary term takes the form
\begin{equation}
\Lfd(\theta)=\frac{1}{|X_h|}\sum_{x_i\in X_h}\left[\bigl(D_x^h\mathbf{r}_\theta\bigr)_i^2+\bigl(D_y^h\mathbf{r}_\theta\bigr)_i^2\right].
\end{equation}
The total loss becomes
\begin{equation}
\Ltot(\theta)=\Lpde(\theta)+\lambda_{\mathrm{BC}}\Lbc(\theta)+\lambda_{\mathrm{aux}}\Lfd(\theta).
\end{equation}

This is not the same as replacing the governing residual by a discrete PDE operator on $u_\theta$. The main residual is still continuous. The discrete operator acts only on the \emph{sampled} residual field. The next observation is intentionally modest: it is a compatibility check, not a convergence theorem. Its role is simply to show that the auxiliary term does not penalize an exact pointwise solution merely because that term is assembled from finite differences on an auxiliary grid.

\begin{proposition}[Exact-solution compatibility of the auxiliary term]
Suppose $u^\star$ solves the continuous PDE pointwise so that $R[u^\star](x)=0$ for all $x\in\Omega$, and suppose the finite-difference operators $D_j^h$ are linear and annihilate constants. Then for any auxiliary grid $X_h$,
\begin{equation}
\Lfd[u^\star]=0.
\end{equation}
\end{proposition}

\begin{proof}
Sampling the exact residual on the auxiliary grid yields the zero vector. Applying any linear difference operator that annihilates constants to the zero vector gives zero. Hence every summand in $\Lfd$ vanishes.
\end{proof}

\begin{remark}
Proposition~1 establishes only an exact-solution compatibility property. It does not claim that minimizing $\Lfd$ yields the exact continuum solution, nor does it provide a convergence rate with respect to the auxiliary spacing $h$. Its practical value is narrower but still useful: unlike a direct high-order smoothness penalty on $u_\theta$, the residual-field term vanishes on any exact pointwise solution of the governing equation. A raw biharmonic penalty on $u_\theta$, by contrast, is not generally minimized by the exact solution of a generic Poisson problem.
\end{remark}

\subsection{AD residual-gradient comparator and weight matching}

The AD comparator keeps the same residual field but differentiates it by AD rather than FD:
\begin{equation}
\Lad(\theta)=\frac{1}{|X_h|}\sum_{x_i\in X_h}\|\grad R_\theta(x_i)\|^2.
\end{equation}
Because $\Lfd$ and $\Lad$ are on different raw scales, the Stage-1 experiments match the AD coefficient to a chosen FD anchor weight. If
\begin{equation}
S_{\mathrm{FD}}=\mathbb{E}[\|\grad_h R\|^2],
\qquad
S_{\mathrm{AD}}=\mathbb{E}[\|\grad R\|^2],
\end{equation}
on the same checkpoint and auxiliary grid, then the matched AD coefficient is chosen as
\begin{equation}
\lambda_{\mathrm{AD,match}}
=\lambda_{\mathrm{FD,anchor}}\,\frac{S_{\mathrm{FD}}}{S_{\mathrm{AD}}}.
\end{equation}
This keeps the AD and FD regularizers comparable at switch-on without pretending that their raw values live on the same scale.

\subsection{Scheduling}

Both Stage 1 and Stage 2 support a ramp-hold-decay schedule for the auxiliary weight $\lambda_{\mathrm{aux}}$. Let $\lambda_0$ denote the target weight. The live auxiliary weight is zero before a start epoch, ramps linearly to $\lambda_0$, optionally holds at $\lambda_0$, then decays either linearly or by cosine to a final fraction $\rho\lambda_0$. In Stage~1, such schedules remain scientifically useful because they help map the field-versus-residual trade-off. In Stage~2, an earlier three-seed schedule ablation suggested that a stronger-early/weaker-late shell can help on some seeds, but the expanded six-seed evidence supports the fixed shell as the stable default configuration.

\subsection{Stage 2: body-fitted shell regularization}

The Stage-2 domain is the cylindrical annulus introduced in Section~\ref{sec:benchmarks},
\begin{equation}
\Omega_{\mathrm{ann}}=\{(r,\theta,z): r_{\min}<r<r_o(\theta),\; 0<z<L\},
\qquad
r_o(\theta)=r_{\max}+0.25\,r_{\max}\sin(3\theta).
\end{equation}
To avoid extending the auxiliary grid outside the physical domain, we introduce a computational radial coordinate
\begin{equation}
s=\frac{r-r_{\min}}{r_o(\theta)-r_{\min}}\in[0,1],
\qquad
r(s,\theta)=r_{\min}+s\bigl(r_o(\theta)-r_{\min}\bigr).
\end{equation}
The main Stage-2 shell occupies $s\in[0.75,0.98]$ and is discretized by shell-bank nodes
\begin{equation}
\{s_i\}_{i=1}^{n_s},\qquad \{\theta_j\}_{j=1}^{n_\theta},\qquad \{z_k\}_{k=1}^{n_z},
\end{equation}
with $(n_s,n_\theta,n_z)=(8,32,32)$ in the default experiments. The indices are intentionally concrete: $i$ counts radial shell lines, $j$ counts azimuthal nodes, and $k$ counts axial nodes.

At those shell nodes we still compute the continuous PDE residual by AD. Writing
\begin{equation}
R_{ijk}=R_\theta\bigl(r(s_i,\theta_j),\theta_j,z_k\bigr)=\Delta_c T_\theta\bigl(r(s_i,\theta_j),\theta_j,z_k\bigr),
\end{equation}
where $\Delta_c$ is the cylindrical Laplacian, the shell regularizer is
\begin{equation}
\mathcal{L}_{\mathrm{shellRG}}
=
\frac{\sum_{i,j,k} w_{ijk}\left[
\bigl(D_r^h R\bigr)_{ijk}^2+
\left(\frac{D_\theta^h R}{h_\theta}\right)_{ijk}^2+
\bigl(D_z^h R\bigr)_{ijk}^2
\right]}
{\sum_{i,j,k} w_{ijk}}.
\end{equation}
Here $D_r^h$, $D_\theta^h$, and $D_z^h$ are finite-difference operators along the three shell directions; $w_{ijk}$ are the shell weights (quadrature weights in the main experiments); and
\begin{equation}
h_\theta(s,\theta)=\sqrt{r(s,\theta)^2+\left(s\,\frac{dr_o}{d\theta}\right)^2}
\end{equation}
converts angular differences into physical tangential distance. In the implementation, the radial operator uses the local physical shell spacing, and the main shell average excludes the shell boundary lines rather than inventing off-domain ghost values.

This is the point where the method differs most clearly from a direct discrete-PDE replacement. The residual field itself remains continuous and AD-based. The structured shell contributes only a localized FD probe of how that residual field varies near the outer wall.

For readability, Tables~\ref{tab:stage1-main} and~\ref{tab:stage2-main} use short model labels.
\emph{OFF} denotes the baseline model without the auxiliary regularizer.
In Stage~1, \emph{FD} denotes the finite-difference residual-gradient regularizer, whereas
\emph{AD} denotes the automatic-differentiation residual-gradient baseline.
The qualifier \emph{fixed} indicates that the auxiliary weight $\lambda_{\mathrm{aux}}$ is held constant after activation,
whereas \emph{linear} indicates a linearly scheduled weight that is ramped up after activation,
held for a prescribed interval, and then decayed to a smaller late-training value.
In Stage~2, \emph{shell fixed} and \emph{shell linear} denote the corresponding body-fitted
outer-shell variants.
To keep the optimizer discussion compact, we write \emph{Adam95} and \emph{Adam999}
for Adam with $\beta_2=0.95$ and $\beta_2=0.999$, respectively.
We write \kourbeta{} for the Kourkoutas-$\beta$ optimizer family used in the inherited PINN3D benchmark; when we refer to the \kourbeta{} regime, we mean this optimizer family together with the associated \texttt{kbeta\_decay} setting and the learning-rate schedule specified for the corresponding experiment.

\section{Stage 1: controlled Poisson benchmark}

\subsection{Setup}

The Stage-1 PDE is the manufactured Poisson problem on $[0,1]^2$ with exact solution
\begin{equation}
u^\star(x,y)=\sin(\pi x)\sin(\pi y)+0.2\sin(3\pi x)\sin(2\pi y).
\end{equation}
The main experiment uses a 6-layer width-96 \texttt{tanh} MLP. Model initialization, training cloud, and validation cloud are seeded separately, and all reported 30k results evaluate the best-validation checkpoint on a fresh Sobol audit together with shifted-grid checks. The core replicated comparison keeps five arms: plain PINN, fixed FD $10^{-3}$, scheduled FD $10^{-3}\to10^{-4}$, fixed AD matched from $10^{-3}$, and scheduled AD matched from the same anchor. The role of Stage~1 is therefore not to crown a universal winner, but to map the field-versus-residual trade-off under controlled conditions.

\subsection{Main three-seed result}

Stage~1 is best read as a controlled mechanism study. The comparison includes the baseline PINN, fixed and scheduled FD residual-gradient regularizers, and a matched AD residual-gradient baseline, all evaluated under the same training and audit protocol. To check that the observed gains are not tied to a particular auxiliary-grid phase, Stage~1 also includes anti-grid-lock ablations based on alternative auxiliary-grid constructions together with shifted-grid and fresh-cloud audits; these are summarized briefly in the main text and reported in detail in \ref{sec:App_Stage1_ablations} (see in particular Table~\ref{tab:anti-grid-lock} and the accompanying discussion).

The main 30k Stage-1 comparison is shown in Table~\ref{tab:stage1-main}. To avoid ambiguity, the table mixes two different kinds of quantities. The column \emph{best val} is the held-out validation-cloud objective used for model selection, whereas the relative $L^2$ and RMSE columns are computed on the fresh Sobol audit against the manufactured solution and residual. Stage~1 therefore does not use an external numerical reference solver. The results make two points at once. First, every regularized arm improves mean best validation and mean fresh-cloud residual metrics over the plain baseline. Second, the ranking is genuinely multi-objective rather than one-dimensional. Fixed AD is the strongest mean residual cleaner; scheduled AD is the best mean field approximant; fixed and scheduled FD remain strong simpler alternatives.

\begin{table}[t]
\centering
\caption{Core 30k three-seed Stage-1 comparison on the fresh Sobol audit. Entries are mean $\pm$ sample standard deviation over seeds 0, 1, and 2. In the scaled columns, the numbers shown are mantissas under the power of ten indicated in the header.}
\label{tab:stage1-main}
\small
\setlength{\tabcolsep}{4.5pt}
\renewcommand{\arraystretch}{1.08}
\begin{tabular*}{\linewidth}{@{\extracolsep{\fill}}lccccc@{}}
\toprule
model
& \shortstack[c]{best val\\($\times 10^{-4}$)}
& \shortstack[c]{rel $L^2(u)$\\($\times 10^{-3}$)}
& \shortstack[c]{rel $L^2(\nabla u)$\\($\times 10^{-3}$)}
& \shortstack[c]{residual RMSE\\($\times 10^{-2}$)}
& \shortstack[c]{$\|\nabla R\|$\\RMSE} \\
\midrule
PINN off
& $6.14 \pm 0.30$
& $2.46 \pm 0.50$
& $8.75 \pm 1.79$
& $2.45 \pm 0.10$
& $1.467 \pm 0.175$ \\

FD fixed
& $3.20 \pm 0.68$
& $2.41 \pm 0.35$
& $8.26 \pm 1.12$
& $1.63 \pm 0.21$
& $0.996 \pm 0.052$ \\

FD linear
& $4.20 \pm 0.48$
& $2.29 \pm 0.38$
& $8.08 \pm 1.27$
& $1.98 \pm 0.16$
& $1.162 \pm 0.032$ \\

AD fixed
& $\mathbf{3.01} \pm 0.78$
& $2.49 \pm 0.33$
& $8.24 \pm 1.06$
& $\mathbf{1.49} \pm 0.21$
& $\mathbf{0.821} \pm 0.092$ \\

AD linear
& $3.82 \pm 0.55$
& $\mathbf{2.24} \pm 0.36$
& $\mathbf{7.85} \pm 1.18$
& $1.86 \pm 0.18$
& $1.038 \pm 0.044$ \\
\bottomrule
\end{tabular*}
\normalsize
\end{table}

Figure~\ref{fig:stage1-frontier} shows the resulting mean field-versus-residual frontier. To keep the visual language simple, OFF is shown in blue, FD arms in green, and AD arms in orange; fixed arms use squares and scheduled arms circles. The fixed arms move farthest downward in residual RMSE, whereas the scheduled AD arm moves farthest left in field error. Figure~\ref{fig:stage1-delta} makes the same point relative to the plain baseline: both schedules improve all four fresh-cloud means, while the fixed arms provide the strongest residual reductions. Runtime follows the expected ordering and is reported here in the text rather than in the table: in the three-seed means, OFF is the cheapest arm (about 145~s), FD regularization sits in the middle (about 347~s), and AD regularization is the most expensive (about 464--468~s).

\begin{figure}[t]
    \centering
    \includegraphics[width=0.84\textwidth]{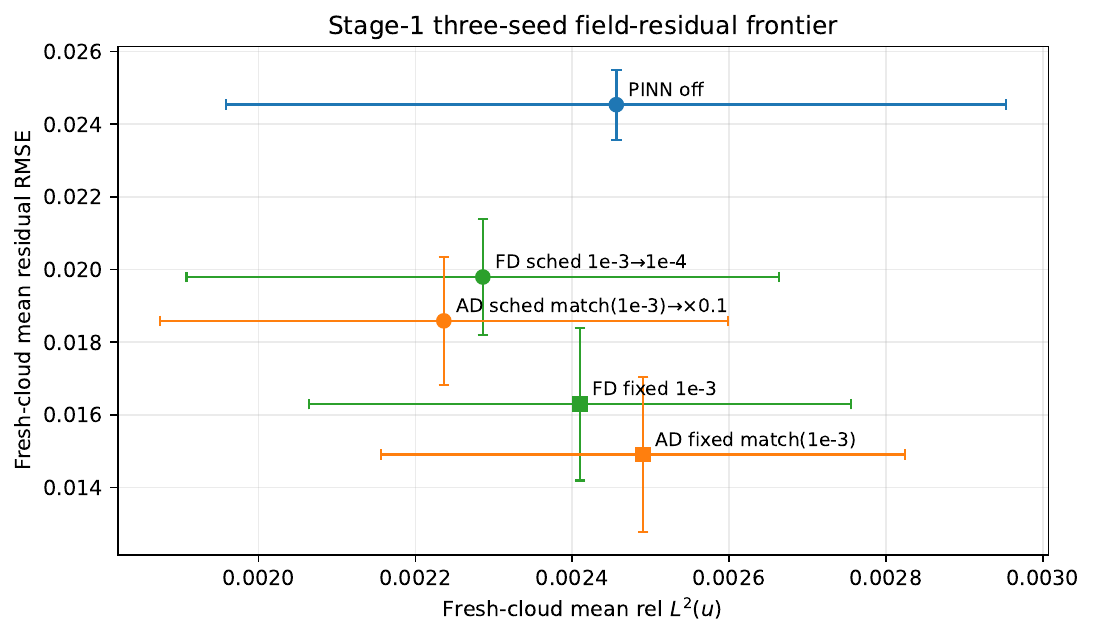}
    \caption{Three-seed mean field-residual frontier for Stage 1. The fixed AD and fixed FD runs (squares) achieve the lowest residual RMSE, while the scheduled runs (circles) achieve the best field accuracy among the core regularized arms.}
    \label{fig:stage1-frontier}
\end{figure}

\begin{figure}[t]
    \centering
    \includegraphics[width=0.84\textwidth]{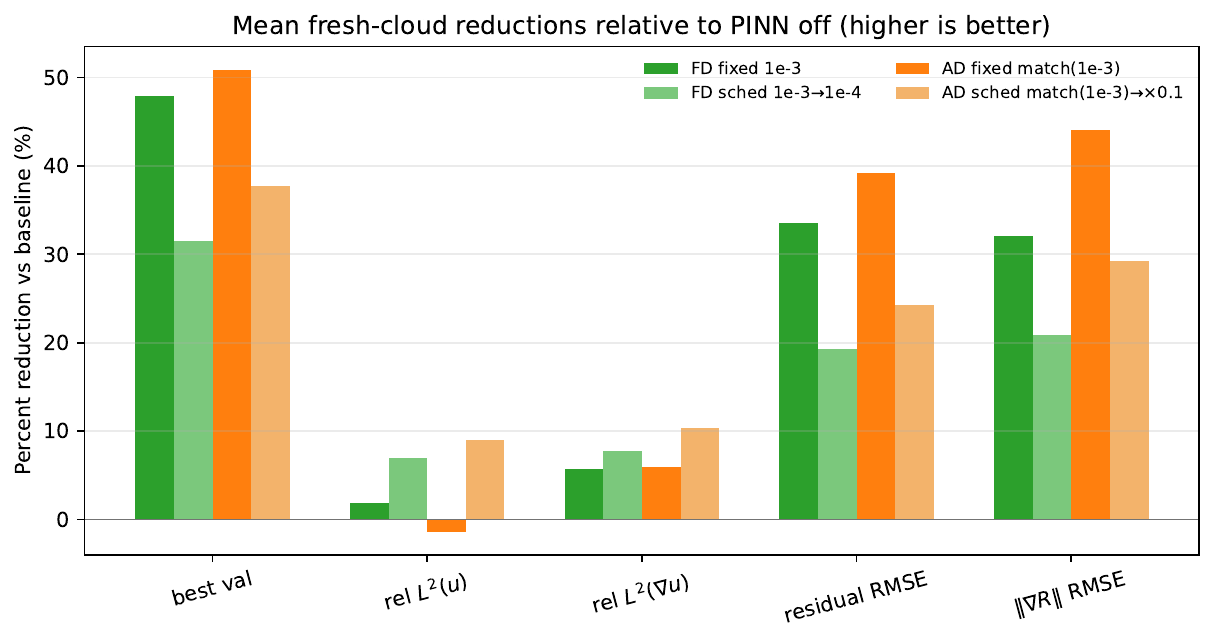}
    \caption{Mean fresh-cloud error reductions relative to the plain PINN baseline. Positive bars mean lower mean error than the baseline. Both schedules improve all four fresh-cloud metrics shown here.}
    \label{fig:stage1-delta}
\end{figure}

\subsection{What Stage 1 does and does not establish}

Stage~1 should be read as a controlled mechanism study rather than as a competition intended to identify a universal winner. Across the three reported seeds, the regularized runs move the solution along a reproducible trade-off between field accuracy and residual cleanliness, and the accompanying anti-grid-lock checks (\ref{sec:App_anti_grid_lock}) do not indicate obvious sensitivity to a single auxiliary-grid phase. What Stage~1 does \emph{not} establish is uniform superiority of the FD residual-gradient regularizer over the matched AD residual-gradient baseline. In the present experiments, the fixed AD run remains the strongest mean residual cleaner and mean validation performer. The more informative conclusion is therefore that Stage~1 provides a controlled map of the mechanism and of the resulting Pareto trade-off. Within that picture, the FD auxiliary regularizer emerges as a competitive alternative to an AD residual-gradient regularizer, with the practical advantages of lower derivative order and a more natural compatibility with the discrete language of a surrounding numerical solver.

\section{Stage 2: body-fitted shell regularization in PINN3D}

\subsection{Problem and evaluation protocol}

We now turn from the benchmark description to the specific Stage-2 comparison. All Stage-2 runs use the same inherited PINN3D workload summarized in Table~\ref{tab:stage2-protocol}. The baseline optimizer regime is \kourbeta{} with \texttt{kbeta\_decay=0.98}. Throughout the Stage~2 optimizer study we use two cosine learning-rate schedules with the same final value of $10^{-5}$. The inherited schedule starts from an initial learning rate of $7.5\times 10^{-3}$; when we need a short label, we refer to it as the \emph{default LR} schedule. The milder schedule starts from $10^{-3}$; when a short label is convenient, we refer to it as the \emph{low-LR} schedule.

The Stage-2 design space considered in this paper includes three arms:
\begin{enumerate}[leftmargin=1.5em]
    \item \textbf{OFF}: the base PINN3D objective with no shell regularizer;
    \item \textbf{fixed shell}: a constant shell residual-gradient weight of $5\times 10^{-4}$;
    \item \textbf{scheduled shell}: a linear shell-weight schedule that starts at $10^{-3}$ and decays to $5\times 10^{-4}$ after the hold interval.
\end{enumerate}

The key modeling choice is that evaluation is not reduced to one scalar loss. In Stage~2 the shell is intended to help exactly where the baseline model struggles most: at the outer wall. We therefore treat the dense wall audits as the primary evidence and interpret the scalar objective as supporting information rather than as the sole model-selection criterion. This choice becomes important later, because some optimizer/LR variants improve a scalar validation total without being the best on the wall-flux quantities that motivated the shell in the first place.

\subsection{Six-seed 100k main result}

Initial three-seed schedule ablations suggested that a linearly decayed shell weight could be competitive on some seeds, but that variant was not consistent enough to motivate an expanded six-seed main sweep. The principal Stage~2 comparison therefore focuses on the two most informative arms in the main \kourbeta{} regime: \emph{OFF} and a \emph{fixed shell} with weight $5\times 10^{-4}$.

\begin{table}[t]
\centering
\caption{Stage-2 six-seed 100k comparison on the wavy-annulus benchmark under the main \kourbeta{} regime. Entries are mean $\pm$ sample standard deviation over seeds 0--5. In the scaled columns, the numbers shown are mantissas under the power of ten indicated in the header.}
\label{tab:stage2-main}
\small
\setlength{\tabcolsep}{4.8pt}
\renewcommand{\arraystretch}{1.08}
\begin{tabular*}{\linewidth}{@{\extracolsep{\fill}}lcccc@{}}
\toprule
model
& \shortstack[c]{final loss\\($\times 10^{-5}$)}
& \shortstack[c]{wall BC RMSE\\($\times 10^{-3}$)}
& \shortstack[c]{$dT/dn$ RMSE\\($\times 10^{-3}$)}
& \shortstack[c]{$T_{\mathrm{wall}}$ RMSE\\($\times 10^{-2}$)} \\
\midrule
OFF
& $2.86 \pm 5.94$
& $12.17 \pm 9.00$
& $9.21 \pm 6.09$
& $5.06 \pm 5.55$ \\

shell fixed
& $\mathbf{0.231} \pm 0.351$
& $\mathbf{0.929} \pm 0.403$
& $\mathbf{0.963} \pm 0.418$
& $\mathbf{1.035} \pm 0.009$ \\
\bottomrule
\end{tabular*}
\normalsize
\end{table}

Relative to OFF, the fixed shell reduces the mean wall BC RMSE by about a factor of 13 and the mean wall-flux RMSE by about a factor of 10. Across the six paired seeds it improves those two primary Stage~2 metrics on every seed, improves wall-temperature RMSE on five of the six seeds, and improves the final scalar objective on four of the six seeds. An exact paired sign test therefore gives a two-sided $p$-value of $0.031$ for the wall-BC and wall-flux comparisons, whereas the corresponding win counts for wall temperature (5/6) and final scalar loss (4/6) are less decisive. Appendix Table~\ref{tab:app-stage2-kbeta-sweep} lists the six paired \kourbeta{} runs explicitly. Runtime is again discussed here rather than in the table: the fixed shell costs about $3.27\times 10^{-3}$ minutes per epoch, versus about $1.55\times 10^{-3}$ for OFF, so the shell roughly doubles the epoch cost in exchange for substantially better wall-facing behavior. Figure~\ref{fig:stage2-delta} summarizes the mean reductions relative to OFF, and Figure~\ref{fig:stage2-seedwise} shows the six-seed paired comparison on the two primary wall-facing metrics.

\subsection{How Stage 2 should be interpreted}
Stage~2 should be interpreted as a targeted improvement result rather than as a claim of uniform improvement in every scalar objective. The shell regularizer was introduced to improve the outer-wall flux behavior, and the Stage~2 evidence should therefore be judged primarily by flux- and boundary-condition-facing metrics rather than by wall temperature alone or by a single scalar validation objective. Seeds~2 and~4 make this point especially clearly: in both cases the OFF run attains the smaller final scalar loss, yet the fixed shell produces substantially better outer-wall flux and boundary-condition behavior. For the present application, the natural model-selection hierarchy is therefore
\begin{enumerate}[leftmargin=1.5em]
    \item wall-flux RMSE, $\mathrm{RMSE}(\partial_n T_{\mathrm{wall}})$, against the FD reference,
    \item outer-wall boundary-condition residual RMSE against the FD reference,
    \item outer-wall BC audit RMSE on the independent wall diagnostic grid,
    \item scalar loss and validation loss,
    \item wall-temperature RMSE as a secondary field metric.
\end{enumerate}
Under that hierarchy, the fixed shell with weight $5\times10^{-4}$ is the most stable Stage~2 default among the tested configurations.

\begin{figure}[h]
    \centering
    \includegraphics[width=0.84\textwidth]{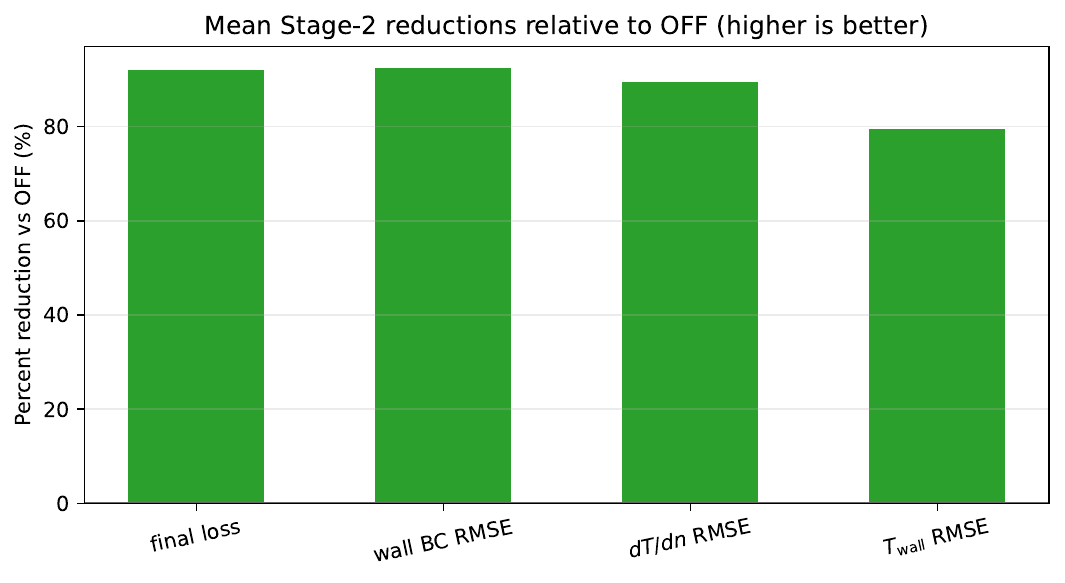}
    \caption{Mean Stage-2 reductions relative to the OFF baseline across seeds 0--5. The fixed shell improves all four metrics shown here, with the largest relative gains on the wall-flux and wall-BC quantities that motivate the shell construction.}
    \label{fig:stage2-delta}
\end{figure}

\begin{figure}[h]
    \centering
    \includegraphics[width=0.84\textwidth]{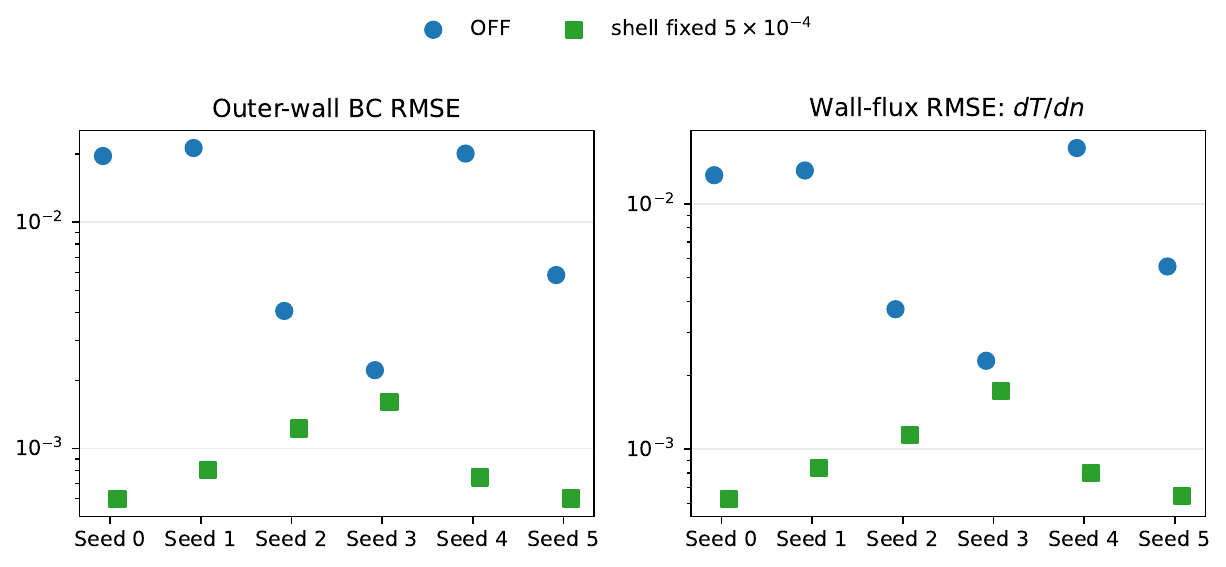}
    \caption{Six-seed Stage-2 comparison on the two primary wall-facing metrics under the main \kourbeta{} regime. Seeds are shown as discrete replicate categories rather than along a continuous axis. The fixed shell (squares) improves both the wall BC RMSE and the wall-flux RMSE on all six seeds.}
    \label{fig:stage2-seedwise}
\end{figure}

\subsection{Optimizer and learning-rate sensitivity}

The primary Stage~2 claim is centered on the \kourbeta{} regime because earlier work identified it as the strongest optimizer family on the PINN3D benchmark family~\cite{Kassinos2025KBeta}. However, we have also performed a  optimizer study for a narrower purpose: to test whether the shell effect survives outside that regime and to distinguish optimizer failures from failures of the regularizer itself.

Under the inherited Stage~2 cosine schedule with initial learning rate $7.5\times10^{-3}$, neither Adam95 nor Adam999 provides a reliable baseline: the OFF runs diverge, and the fixed-shell counterparts are less catastrophic but still unusable. This identifies the inherited schedule as appropriate for \kourbeta{} on this benchmark, but too aggressive for standard Adam. Appendix Table~\ref{tab:app-stage2-optlr-seed2} records the seed-2 optimizer and learning-rate matrix that highlights the failure mode of standard Adam.

When the initial learning rate is lowered to $10^{-3}$ Adam999 becomes usable.
Appendix Table~\ref{tab:app-stage2-adam999-sweep} summarizes the completed Adam999 sweep with the milder initial learning rate $10^{-3}$ across seeds~0--5. Relative to OFF, the fixed shell lowers the median final loss from $1.13\times10^{-4}$ to $2.96\times10^{-5}$, lowers the median wall BC RMSE from $1.10\times10^{-2}$ to $2.98\times10^{-3}$, and lowers the median wall-flux RMSE from $9.70\times10^{-3}$ to $3.05\times10^{-3}$. It also improves wall-temperature RMSE on five of the six paired seeds.

However, the seedwise picture is more mixed than the medians alone suggest. On the two primary wall-facing metrics, i.e. wall-flux RMSE and wall-BC RMSE, the fixed shell is better on three of the six Adam999 pairs and worse on the other three. An exact paired sign test therefore gives no evidence of a consistent per-seed advantage in this regime ($p=1.0$, two-sided for both metrics). 

These results suggest that the shell effect is not exclusive to \kourbeta{}, but nevertheless the performance of the method is sensitive to the choice of optimizer and learning rate schedules. That is why the fixed-shell \kourbeta{} configuration remains the recommended default in the present benchmark family.

\section{Discussion}

The two-stage structure of the paper is intended to separate mechanism from application. Stage~1 asks whether a structured finite-difference probe of the residual field changes optimization in a controlled and interpretable way. Stage~2 then asks whether the same idea remains useful once it is transplanted into a realistic three-dimensional benchmark with a clear physical target. Taken together, the two stages support a practical message: the hybrid regularizer is most compelling when it is aligned with the quantity that the user actually wishes to improve.

The results also clarify the practical role of the proposed hybrid design. The method does not replace the continuous PINN residual by a discrete one. Instead, it preserves the AD-based governing residual and adds a structured finite-difference probe only through an auxiliary term on the sampled residual field. The significance of that design is therefore best understood through the quantities it improves in practice, which is the focus of the discussion below.

\subsection{Choosing the method for the metric that matters}

One lesson that emerges especially clearly from Stage~2 is that model selection depends on the quantity of interest. In the annular benchmark studied here, the shell regularizer is most convincing when the evaluation focuses on outer-wall flux consistency and boundary-condition behavior. Those are the metrics for which the body-fitted shell was designed, and those are the metrics on which it delivers the most robust gains across seeds. By contrast, scalar objective values, scalar validation values, and wall-temperature errors can tell a more mixed story. A run can therefore look attractive under one summary statistic and less attractive under another without any contradiction.

This is not a weakness of the method; it is a reminder that inverse or physics-informed models should be judged by the physical quantity that motivated the regularizer in the first place. In the present problem, that quantity is the outer-wall flux behavior, not the smallest possible value of a single generic scalar objective. For that reason, the paper treats the wall-flux and wall-BC diagnostics as the primary Stage~2 selection criteria, and interprets aggregate loss values as supporting rather than decisive evidence.

The optimizer study adds one more practical qualification. The shell effect is not exclusive to \kourbeta{}, but its stability depends on the optimizer and learning-rate regime. As shown in Section~6 and Appendix Tables~\ref{tab:app-stage2-optlr-seed2} and~\ref{tab:app-stage2-adam999-sweep}, Adam95 and Adam999 are unreliable under the inherited Stage~2 learning-rate schedule, whereas Adam999 becomes usable after the initial learning rate is reduced to $10^{-3}$. Even in that milder regime, however, the shell benefit is less uniform across seeds than in the default \kourbeta{} setting. The optimizer results therefore support the same practical conclusion as the main Stage~2 sweep: the fixed-shell \kourbeta{} configuration remains the preferred default in the present benchmark family.

\subsection{Reproducibility note}

Unless otherwise noted, all reported experiments in the paper were generated on a Mac Studio with M3 Ultra, using macOS 26.4 (25E246), Python 3.13.11, and MLX 0.31.1. The
matched-seed comparisons were kept on a single hardware/software stack.

\section{Conclusion}

This paper studies a hybrid PINN design in which finite differences are introduced only through an auxiliary regularizer on the sampled residual field, while the governing residual itself remains AD-based. Stage~1 shows that this auxiliary residual-gradient idea has a measurable effect on training and provides a controlled map of the trade-off between field accuracy and residual cleanliness. Stage~2 shows that the same idea can be transplanted into a realistic three-dimensional benchmark by localizing the regularizer to a body-fitted outer shell.

For the PINN3D annulus benchmark considered here, the most important quantities of interest are the outer-wall heat-flux and boundary-condition metrics. In that setting, the fixed shell regularizer produces the clearest and most stable gains. Under the \kourbeta{} regime with the default Stage~2 learning-rate schedule, the fixed shell improves the two pre-specified primary Stage~2 metrics on all six reported seeds. Under an exact paired sign test, this corresponds to $p=0.031$ in the two-sided case ($p=0.016$ one-sided), providing statistical support for improvement on the primary wall-facing quantities of interest.

The optimizer study shows that the shell effect is not exclusive to \kourbeta{}, but it also shows that optimizer and learning-rate choices matter. With the inherited, more aggressive Stage~2 learning-rate schedule, Adam95 and Adam999 do not provide reliable baselines on this benchmark. When the initial learning rate is reduced to $10^{-3}$, Adam999 becomes usable and the fixed shell often improves performance, but the gain is not uniformly robust across seeds on the primary wall-flux and wall-BC metrics. The strongest and most stable Stage~2 results therefore remain those obtained with \kourbeta{}.

Taken together, the results support a focused conclusion. An auxiliary FD regularizer applied to the sampled AD residual field can be a useful addition to a PINN when it is aligned with a clearly defined quantity of interest and deployed where that quantity is most sensitive. In the present benchmark family, that role is played by the body-fitted outer shell and the corresponding outer-wall flux behavior.

\appendix

\section{Finite-difference reference solution for Stage~2}\label{app:fd-reference}

The Stage~2 wall-reference metrics are anchored to an independent steady finite-difference (FD) solution of the same annular heat-conduction problem. This reference solve is used as a shared numerical anchor for field quality, not as a claim of exact truth. Its role is comparative: it tells us whether distinct PINN runs produce materially different wall temperature and wall-flux behavior on the same underlying problem. The FD solution is computed on a cylindrical grid using the companion solver to be released with the manuscript artifacts.

The practical wall reference used in the paper is extracted from the $(25,48,193)$ FD solution in $(N_s,N_\theta,N_z)$. This choice intentionally prioritizes axial resolution, because the prescribed outer-wall flux is a nonuniform function of $z$ rather than a spatially uniform boundary load. The wall-reference diagnostics used in Stage~2 are extracted on the outer wall and correspond to the quantities denoted in the main text by $T_{\mathrm{wall}}$, $\partial_n T_{\mathrm{wall}}$, and the induced wall boundary-condition residual. These quantities differ conceptually from the \emph{outer-wall BC audit}: the audit is evaluated directly from the PINN prediction against the prescribed flux, whereas the FD wall-reference metrics compare the PINN prediction against an independent numerical solution.

A limited grid-resolution study was used to assess whether this FD field is sufficiently resolved for the comparative role it plays in the paper. In the axial direction, refinement at fixed $N_s$ and $N_\theta$ changes the solution only weakly. In the radial direction, wall temperature is more sensitive than wall flux, but the wall-flux slice is already effectively converged for practical purposes on the finer grids. In particular, comparing the manuscript anchor $(25,48,193)$ against a radially refined $(193,48,193)$ solution changes the wall-flux slice only negligibly ($dTdn_{\mathrm{wall}}$ RMSE $=5.63\times10^{-6}$, max-abs $=2.15\times10^{-5}$), while the wall-temperature slice is more sensitive ($T_{\mathrm{wall}}$ RMSE $=9.65\times10^{-3}$). Importantly, the main Stage~2 conclusions are not supported by the FD wall-reference metrics alone: the same qualitative shell-versus-OFF advantage is also visible in the independent outer-wall BC audit. Table~\ref{tab:app-fd-grid-study} summarizes the checks used here.

\begin{table}[h]
\centering
\caption{Summary of the FD grid checks used to calibrate the Stage~2 wall-reference solution. Relative $L_2$ changes are computed on common grid points after refinement. For the wall-reference slice, the most important quantity is the wall-normal flux; the corresponding radial-refinement changes are substantially smaller than the shell-versus-OFF differences reported in the main Stage~2 results.}
\label{tab:app-fd-grid-study}
\small
\begin{tabular}{llll}
\toprule
direction & comparison & relative $L_2$ change & max pointwise change \\
\midrule
axial & $(97,48,49) \to (97,48,97)$ & $1.62\times10^{-4}$ & $6.66\times10^{-4}$ \\
axial & $(25,48,49) \to (25,48,97)$ & $1.58\times10^{-4}$ & --- \\
axial & $(25,48,97) \to (25,48,193)$ & $4.61\times10^{-5}$ & --- \\
radial ($T_{\mathrm{wall}}$) & $(25,48,97) \to (193,48,97)$ & $1.47\times10^{-2}$ & $1.79\times10^{-2}$ \\
radial ($\partial_n T_{\mathrm{wall}}$) & $(25,48,97) \to (193,48,97)$ & $1.71\times10^{-5}$ & $2.06\times10^{-5}$ \\
radial ($T_{\mathrm{wall}}$) & $(25,48,193) \to (193,48,193)$ & $1.47\times10^{-2}$ & $1.79\times10^{-2}$ \\
radial ($\partial_n T_{\mathrm{wall}}$) & $(25,48,193) \to (193,48,193)$ & $1.74\times10^{-5}$ & $2.15\times10^{-5}$ \\
\bottomrule
\end{tabular}
\end{table}

\eject
\section{Additional Stage-2 tables}

This appendix collects the raw per-seed Stage~2 tables that support the main comparisons in the text. Table~\ref{tab:app-stage2-optlr-seed2} reinstates the seed-2 optimizer and learning-rate sensitivity matrix that first revealed the instability of the inherited learning-rate schedule for standard Adam. Tables~\ref{tab:app-stage2-kbeta-sweep} and~\ref{tab:app-stage2-adam999-sweep} list the completed six-seed paired sweeps for the main \kourbeta{} regime and for the low-learning-rate Adam999 portability study.

\begin{table}[h]
\centering
\caption{Seed-2 optimizer and initial-learning-rate sensitivity for Stage~2, comparing OFF and the fixed shell. The inherited schedule starts from $7.5\times 10^{-3}$ and decays to $10^{-5}$; the milder probes use the same cosine form but start from $10^{-3}$.}
\label{tab:app-stage2-optlr-seed2}
\resizebox{\textwidth}{!}{%
\begin{tabular}{ll l c c c c}
\toprule
optimizer & init. LR & arm & final loss & wall BC RMSE & $dT/dn$ RMSE & $T_{\mathrm{wall}}$ RMSE \\
\midrule
\kourbeta{} & $7.5\times 10^{-3}$ & OFF & $5.048\times10^{-7}$ & $4.053\times10^{-3}$ & $3.722\times10^{-3}$ & $5.675\times10^{-2}$ \\
\kourbeta{} & $7.5\times 10^{-3}$ & fixed shell $5\times10^{-4}$ & $6.695\times10^{-7}$ & $1.225\times10^{-3}$ & $1.144\times10^{-3}$ & $1.033\times10^{-2}$ \\
\kourbeta{} & $10^{-3}$ & OFF & $1.052\times10^{-5}$ & $2.003\times10^{-3}$ & $2.076\times10^{-3}$ & $1.056\times10^{-2}$ \\
\kourbeta{} & $10^{-3}$ & fixed shell $5\times10^{-4}$ & $4.751\times10^{-6}$ & $1.392\times10^{-3}$ & $1.446\times10^{-3}$ & $1.059\times10^{-2}$ \\
Adam95 & $7.5\times 10^{-3}$ & OFF & $1.576\times10^{25}$ & $8.533\times10^{7}$ & $6.322\times10^{8}$ & $5.762\times10^{5}$ \\
Adam95 & $7.5\times 10^{-3}$ & fixed shell $5\times10^{-4}$ & $2.614\times10^{-1}$ & $3.232\times10^{-1}$ & $3.229\times10^{-1}$ & $5.748\times10^{-1}$ \\
Adam999 & $7.5\times 10^{-3}$ & OFF & $3.465\times10^{6}$ & $6.149\times10^{1}$ & $6.149\times10^{1}$ & $3.083\times10^{2}$ \\
Adam999 & $7.5\times 10^{-3}$ & fixed shell $5\times10^{-4}$ & $3.261\times10^{-1}$ & $2.559\times10^{-1}$ & $2.554\times10^{-1}$ & $5.160\times10^{-1}$ \\
Adam999 & $10^{-3}$ & OFF & $3.195\times10^{-4}$ & $2.100\times10^{-2}$ & $2.017\times10^{-2}$ & $2.138\times10^{-1}$ \\
Adam999 & $10^{-3}$ & fixed shell $5\times10^{-4}$ & $1.385\times10^{-5}$ & $2.232\times10^{-3}$ & $2.338\times10^{-3}$ & $1.028\times10^{-2}$ \\
\bottomrule
\end{tabular}}
\end{table}

\begin{table}[h]
\centering
\caption{Completed six-seed Stage~2 sweep under the main \kourbeta{} regime, comparing OFF and the fixed shell.}
\label{tab:app-stage2-kbeta-sweep}
\resizebox{\textwidth}{!}{%
\begin{tabular}{c c c c c c c c c}
\toprule
\multirow{2}{*}{seed} &
\multicolumn{2}{c}{final loss} &
\multicolumn{2}{c}{wall BC RMSE} &
\multicolumn{2}{c}{$dT/dn$ RMSE} &
\multicolumn{2}{c}{$T_{\mathrm{wall}}$ RMSE} \\
\cmidrule(lr){2-3}\cmidrule(lr){4-5}\cmidrule(lr){6-7}\cmidrule(lr){8-9}
& OFF & fixed shell & OFF & fixed shell & OFF & fixed shell & OFF & fixed shell \\
\midrule
0 & \num{7.662e-7} & \num{3.981e-7} & \num{1.960e-2} & \num{5.987e-4} & \num{1.310e-2} & \num{6.264e-4} & \num{1.136e-2} & \num{1.024e-2} \\
1 & \num{1.495e-4} & \num{1.308e-6} & \num{2.125e-2} & \num{8.006e-4} & \num{1.369e-2} & \num{8.380e-4} & \num{8.920e-3} & \num{1.027e-2} \\
2 & \num{5.048e-7} & \num{6.695e-7} & \num{4.053e-3} & \num{1.225e-3} & \num{3.722e-3} & \num{1.144e-3} & \num{5.675e-2} & \num{1.033e-2} \\
3 & \num{1.019e-5} & \num{9.446e-6} & \num{2.219e-3} & \num{1.605e-3} & \num{2.293e-3} & \num{1.726e-3} & \num{1.049e-2} & \num{1.039e-2} \\
4 & \num{6.170e-7} & \num{1.070e-6} & \num{2.009e-2} & \num{7.450e-4} & \num{1.691e-2} & \num{8.012e-4} & \num{1.521e-1} & \num{1.037e-2} \\
5 & \num{9.965e-6} & \num{9.631e-7} & \num{5.838e-3} & \num{6.016e-4} & \num{5.559e-3} & \num{6.438e-4} & \num{6.385e-2} & \num{1.049e-2} \\
\bottomrule
\end{tabular}}
\end{table}

\begin{table}[H]
\centering
\caption{Completed six-seed Adam999 sweep at initial learning rate $10^{-3}$, comparing OFF and the fixed shell.}
\label{tab:app-stage2-adam999-sweep}
\resizebox{\textwidth}{!}{%
\begin{tabular}{c c c c c c c c c}
\toprule
\multirow{2}{*}{seed} &
\multicolumn{2}{c}{final loss} &
\multicolumn{2}{c}{wall BC RMSE} &
\multicolumn{2}{c}{$dT/dn$ RMSE} &
\multicolumn{2}{c}{$T_{\mathrm{wall}}$ RMSE} \\
\cmidrule(lr){2-3}\cmidrule(lr){4-5}\cmidrule(lr){6-7}\cmidrule(lr){8-9}
& OFF & fixed shell & OFF & fixed shell & OFF & fixed shell & OFF & fixed shell \\
\midrule
0 & \num{2.615e-1} & \num{9.933e-5} & \num{3.232e-1} & \num{6.105e-3} & \num{3.229e-1} & \num{6.055e-3} & \num{5.748e-1} & \num{2.241e-2} \\
1 & \num{4.738e-5} & \num{6.606e-6} & \num{1.409e-2} & \num{1.487e-3} & \num{1.167e-2} & \num{1.538e-3} & \num{2.057e-1} & \num{1.057e-2} \\
2 & \num{3.195e-4} & \num{1.385e-5} & \num{2.100e-2} & \num{2.232e-3} & \num{2.017e-2} & \num{2.338e-3} & \num{2.138e-1} & \num{1.028e-2} \\
3 & \num{3.915e-6} & \num{1.010e-5} & \num{1.268e-3} & \num{1.920e-3} & \num{1.337e-3} & \num{1.964e-3} & \num{1.042e-2} & \num{1.056e-2} \\
4 & \num{1.794e-4} & \num{3.203e-4} & \num{7.926e-3} & \num{1.005e-2} & \num{7.732e-3} & \num{9.996e-3} & \num{1.543e-2} & \num{8.466e-3} \\
5 & \num{1.694e-6} & \num{4.532e-5} & \num{1.008e-3} & \num{3.737e-3} & \num{1.062e-3} & \num{3.766e-3} & \num{1.061e-2} & \num{1.015e-2} \\
\bottomrule
\end{tabular}}
\end{table}

\eject
\section{Supporting Stage-1 ablations}\label{sec:App_Stage1_ablations}

The current paper draft keeps the Stage-1 main text focused on the replicated 30k comparison. Two earlier ablations are still worth recording because they explain how the final design was chosen.

\subsection{10k weight sweep}

By 10k, increasing the FD weight from OFF to $10^{-3}$ monotonically improved best validation and the dense/random residual-centric metrics. This was the first sign that the auxiliary term behaves as an early basin shaper rather than as mere post-hoc smoothing.

\begin{table}[H]
\centering
\caption{10k FD-resgrad sweep on the Stage-1 Poisson benchmark (one matched seed triplet).}
\resizebox{0.80\textwidth}{!}{%
\begin{tabular}{lccccc}
\toprule
FD weight & best val & rel $L^2(u)$ & rel $L^2(\grad u)$ & residual RMSE & $\|\grad R\|$ RMSE \\
\midrule
off & 2.178e-03 & 3.708e-03 & 1.170e-02 & 5.187e-02 & 2.029e+00 \\
1e-5 & 2.142e-03 & 3.691e-03 & 1.163e-02 & 5.139e-02 & 2.009e+00 \\
1e-4 & 1.895e-03 & 3.619e-03 & 1.119e-02 & 4.804e-02 & 1.870e+00 \\
5e-4 & 1.260e-03 & 3.401e-03 & 1.051e-02 & 3.905e-02 & 1.510e+00 \\
1e-3 & 9.767e-04 & 3.316e-03 & 9.882e-03 & 3.390e-02 & 1.324e+00 \\
\bottomrule
\end{tabular}}
\end{table}

\subsection{Anti-grid-lock controls}\label{sec:App_anti_grid_lock}

The phase-safe 20k comparison among \texttt{fixed-safe}, \texttt{cycle4}, and \texttt{jitter4} showed no evidence that phase cycling or jitter rescued generalization. Here \texttt{fixed-safe} denotes a single cropped auxiliary grid whose stencil support lies on the common $64\times64$ interior support shared by all half-cell phase shifts; \texttt{cycle4} cycles deterministically through the four half-cell offsets $(0,0)$, $(h/2,0)$, $(0,h/2)$, and $(h/2,h/2)$ while retaining that common-support crop; and \texttt{jitter4} cycles through four slightly jittered banks built on the same phase-safe support. On this benchmark the simple fixed-safe auxiliary grid remained the strongest default.

\begin{table}[H]
\centering
\caption{20k anti-grid-lock ablation on the phase-safe $64\times64$ family.}
\label{tab:anti-grid-lock}
\resizebox{\textwidth}{!}{%
\begin{tabular}{l l c c c c c}
\toprule
FD weight & grid strategy & best val & rel $L^2(u)$ & rel $L^2(\grad u)$ & residual RMSE & $\|\grad R\|$ RMSE \\
\midrule
1e-3 & fixed-safe & 5.366e-04 & 2.489e-03 & 8.193e-03 & 2.257e-02 & 1.222e+00 \\
1e-3 & cycle4 & 6.893e-04 & 5.196e-03 & 8.315e-03 & 2.559e-02 & 1.191e+00 \\
1e-3 & jitter4 & 6.676e-04 & 3.565e-03 & 8.277e-03 & 2.490e-02 & 1.205e+00 \\
1e-2 & fixed-safe & 3.983e-04 & 3.723e-03 & 1.123e-02 & 1.217e-02 & 7.689e-01 \\
1e-2 & cycle4 & 7.197e-04 & 5.224e-03 & 1.166e-02 & 2.190e-02 & 7.235e-01 \\
1e-2 & jitter4 & 1.457e-03 & 4.258e-03 & 1.155e-02 & 3.540e-02 & 8.042e-01 \\
\bottomrule
\end{tabular}}
\end{table}

\section{Reproducibility materials}

For the present submission, the manuscript bundle contains the source \LaTeX{}, the figure-generation scripts, and the derived summary tables used in the paper. The fuller run-level JSON summaries, companion code, and finite-difference reference-solution artifacts are retained as internal provenance during review and are intended for public release after journal acceptance.
\eject


\begin{thebibliography}{99}

\bibitem{raissi2019pinn}
M.~Raissi, P.~Perdikaris, and G.~E. Karniadakis.
\newblock Physics-informed neural networks: A deep learning framework for solving forward and inverse problems involving nonlinear partial differential equations.
\newblock \emph{Journal of Computational Physics}, 378:686--707, 2019.

\bibitem{wang2021gradpath}
S.~Wang, Y.~Teng, and P.~Perdikaris.
\newblock Understanding and mitigating gradient flow pathologies in physics-informed neural networks.
\newblock \emph{SIAM Journal on Scientific Computing}, 43(5):A3055--A3081, 2021.

\bibitem{wang2022ntk}
S.~Wang, X.~Yu, and P.~Perdikaris.
\newblock When and why PINNs fail to train: A neural tangent kernel perspective.
\newblock \emph{Journal of Computational Physics}, 449:110768, 2022.

\bibitem{yu2022gpinn}
J.~Yu, L.~Lu, X.~Meng, and G.~E. Karniadakis.
\newblock Gradient-enhanced physics-informed neural networks for forward and inverse PDE problems.
\newblock \emph{Computer Methods in Applied Mechanics and Engineering}, 393:114823, 2022.

\bibitem{xiang2022hfdpinn}
Z.~Xiang, W.~Peng, W.~Zhou, and W.~Yao.
\newblock Hybrid finite difference with the physics-informed neural network for solving PDE in complex geometries.
\newblock arXiv:2202.07926, 2022.

\bibitem{chiu2022canpinn}
P.-H. Chiu et al.
\newblock CAN-PINN: A fast physics-informed neural network based on coupled automatic-numerical differentiation.
\newblock \emph{Computer Methods in Applied Mechanics and Engineering}, 395:114909, 2022.

\bibitem{langer2026fdpinn}
A.~Langer.
\newblock The ill-posed foundations of physics-informed neural networks and their finite-difference variants.
\newblock arXiv:2601.07017, 2026.

\bibitem{Kassinos2025KBeta}
S.~C. Kassinos.
\newblock \kourbeta{}: A Sunspike-Driven Adam Optimizer with Desert Flair.
\newblock \emph{arXiv preprint} arXiv:2508.12996, 2025.

\bibitem{Kassinos2025PINN3D}
S.~C.~Kassinos.
\newblock \emph{kbeta-pinn3d v1.0.1: First public release}.
\newblock Zenodo, 2025.
\newblock doi: \url{https://doi.org/10.5281/zenodo.16915164}.
\newblock url: \url{https://github.com/sck-at-ucy/kbeta-pinn3d}.

\bibitem{HybridRepo2026}
S.~C. Kassinos.
\newblock Hybrid PINN / PINN3D companion software and artifact archive.
\newblock Software archive accompanying the present manuscript, 2026. To be replaced 
by DOI once available. 
\end{thebibliography}
\end{document}